\documentclass{svproc}
\usepackage{ graphicx }
\usepackage{ mathrsfs }
\usepackage{ amsmath }
\usepackage{ amssymb }
\usepackage{wrapfig}
\usepackage[bookmarks=true]{hyperref}
\usepackage{cleveref} 
\usepackage{cite}

\usepackage[table,xcdraw]{xcolor}
\hypersetup{
  colorlinks   = true, 
  urlcolor     = blue, 
  linkcolor    = blue, 
  citecolor   = red 
}
\crefname{figure}{Fig.}{Fig.}
\crefname{table}{Table}{Table}
\crefname{wrapfigure}{Fig.}{Fig.}
\crefname{theorem}{Theorem}{Theorems}
\crefname{section}{Section}{Sections}
\crefname{hypothesis}{Hypothesis}{Hypotheses}
\crefname{proposition}{Proposition}{Propositions}
\crefname{algorithm}{Algorithm}{Algorithm}
\crefname{definition}{Definition}{Definition}
\crefname{corollary}{Corollary}{Corollary}

\usepackage{url}

\begin{document}
\mainmatter              
%
\title{Measure Preserving Flows for Ergodic Search \\
in Convoluted Environments}
\titlerunning{Measure Preserving Flows for Ergodic Search in Convoluted Environments}  
%
\author{Albert Xu$^\dagger$, Bhaskar Vundurthy$^\dagger$, Geordan Gutow$^\dagger$, Ian Abraham$^*$, \\
Jeff Schneider$^\dagger$, and Howie Choset$^\dagger$}
\authorrunning{Xu \textit{et. al.}} 
%
\tocauthor{}
\institute{$^\dagger$ Carnegie Mellon University, Pittsburgh PA 15232, USA\\
$^*$Yale University, New Haven CT 06511, USA
\footnote[0]{This research was supported in part by an Intelligence Community Postdoctoral Research Fellowship at Carnegie Mellon University, administered by Oak Ridge Institute for Science and Education through an interagency agreement between the U.S. Department of Energy and the Office of the Director of National Intelligence.
}
}
\maketitle              
\begin{abstract} Autonomous robotic search has important applications in robotics, such as the search for signs of life after a disaster. When \emph{a priori} information is available, for example in the form of a distribution, a planner can use that distribution to guide the search.
Ergodic search is one method that uses the information distribution to generate a trajectory that minimizes the ergodic metric, in that it encourages
the robot to spend more time in regions with high information and proportionally less time in the remaining regions.
Unfortunately, prior works in ergodic search do not perform well in complex environments with obstacles such as a building's interior or a maze. To address this, our work presents a modified ergodic metric using the Laplace-Beltrami eigenfunctions to capture map geometry and obstacle locations within the ergodic metric. Further, we introduce an approach to generate trajectories that minimize the ergodic metric while guaranteeing obstacle avoidance using measure-preserving vector fields.
Finally, we leverage the divergence-free nature of these vector fields to generate collision-free trajectories for multiple agents. 
We demonstrate our approach via simulations with single and multi-agent systems on maps representing interior hallways and long corridors with non-uniform information distribution.
In particular, we illustrate the generation of feasible trajectories in complex environments where prior methods fail. 

\end{abstract}

\section{Introduction}

Robotic exploration has multiple applications including search and rescue missions \cite{murphysearch}, autonomous data gathering/monitoring \cite{dunbabin2012robots}, and surveillance/patrolling \cite{patel2021multi}.
Such exploration problems typically entail coverage path planning \cite{choset2001review,galceran2013review,acar2002morse} where the robot determines a trajectory that visits every point in a given space. 
When \emph{a priori} information is available, not every point needs to be visited, and it is beneficial to focus the robot's exploration on higher information regions.
In this context, ergodic search \cite{mathew2011metrics} is a promising effort to aid the robots to spend more time in higher information regions and proportionally less in low information regions. Unlike uniform coverage approaches exemplified by the lawn mower pattern \cite{mezic2011lawnmover}, ergodic search approaches naturally circle around ``hot spots'' of information, for instance the survivor gathering locations in a search and rescue mission \cite{murphysearch}.

However, virtually all prior efforts with ergodic search presume no or small obstacles in the search region \cite{mathew2011metrics,murphey2013trajopt,miller2015ergodic,ren2022local}.
Efforts towards obstacle avoidance utilize stochastic trajectory optimization \cite{ayvali2017ergodic} to discard candidate trajectories that intersect with obstacles, or place repulsive vector fields around the obstacles \cite{patel2021multi,salman2017multi} to push the robots away.
Alternatively, the authors in \cite{lerch2023safety} use control barrier functions to serve as additional hard constraints for trajectory optimization. While such algorithms do avoid obstacles, they either struggle to converge to a feasible trajectory in complex environments or are unable to navigate around larger non-convex obstacles, like in the environment shown in \cref{fig:atrium_eg_intro}.
In this work, we focus on generating feasible ergodic trajectories for uniform and non-uniform information maps in complex environments for multiple robots, as illustrated in orange, green, and purple in \cref{fig:atrium_eg_intro}.

\begin{wrapfigure}{r}{0.5\textwidth}
    \centering
    \vspace{-0.3in}
    \includegraphics[width=\linewidth]{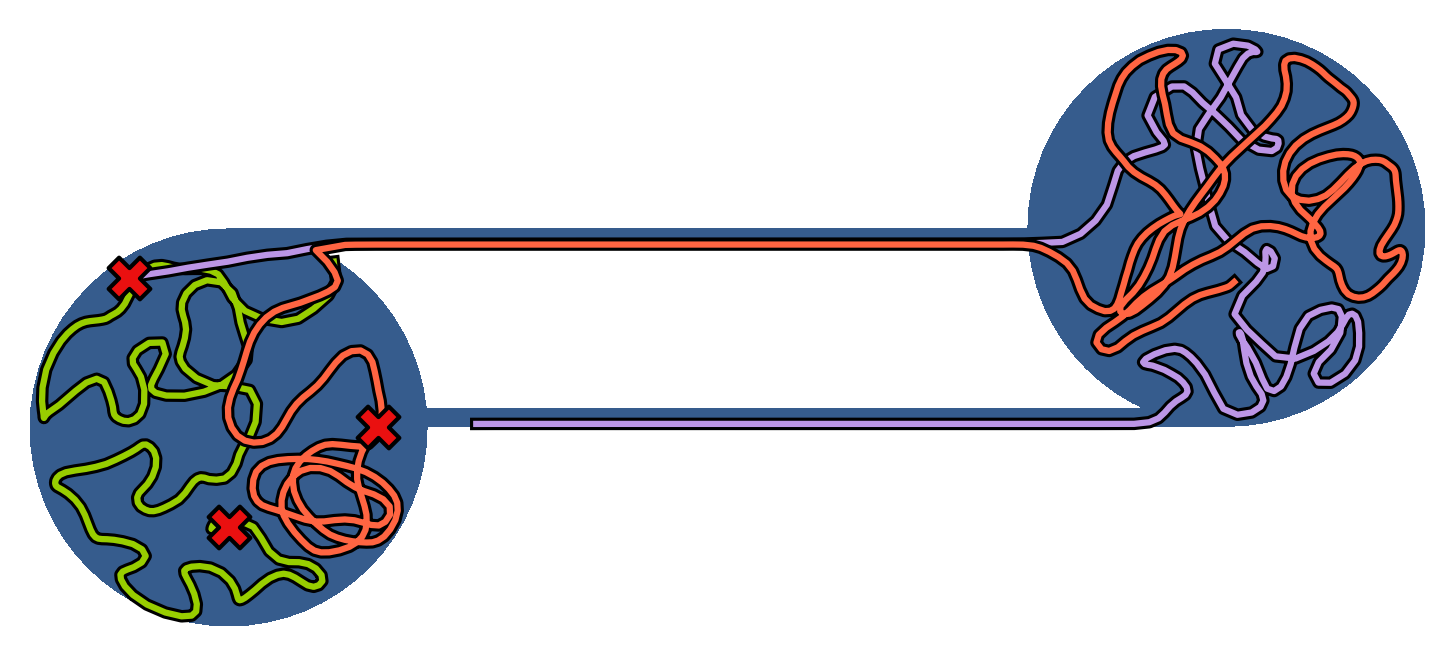}
    \caption{An interior environment with two rooms connected by a narrow hallway. An ergodicity-minimizing trajectory generated using measure-preserving vector fields translates the three robots, starting from the red X's, to the other room through the narrow hallway for efficient information gathering.}
    \vspace{-0.2in}
    \label{fig:atrium_eg_intro}
\end{wrapfigure}

To this end, we first compute measure-preserving vector fields for a given environment with obstacles. Assuming point-sized robots with first-order dynamics, we trace out trajectories by having these point robots flow along the measure-preserving vector fields. Using theorems from ergodic theory, we prove that as time $t\to\infty$, these trajectories are ergodic; the time-averaged statistics of the robot trajectory match the probability distribution of information on the map. Finite-time trajectories can then be generated by seeding a randomly generated fixed-length trajectory and minimizing its ergodicity using gradient-based methods. 

Further, we note that the ergodic metric in prior works relies heavily on the Fourier transforms that are well-defined for Euclidean spaces. While the presence of a few scattered obstacles does not significantly impact the distance metric on Euclidean spaces, more complex environments (see \cref{fig:atrium_eg_intro}) with narrow corridors and sharp turns significantly warp the distance metric, rendering the space non-Euclidean. Given this work primarily focuses on such complex environments, we turn to the Laplace-Beltrami eigenfunctions \cite{levy2006laplace,reuter2009discrete,vallet2008spectral}, which are considered to be the natural generalization of the Fourier basis functions for arbitrary manifolds. These functions naturally encode the information of the free space without the obstacles when generating measure-preserving vector fields. Consequently, the resulting trajectories are always collision-free.
Finally, an added advantage of using measure-preserving ergodic flows is that they are always divergence-free. This allows us to generalize our work to multiple agents without any additional effort and generate collision-free trajectories for all involved agents.

In summary, our contributions are as follows: 
\vspace{-0.05in}
\begin{enumerate}
    \item A formal proof that a system following the dynamics of a time-varying measure-preserving flow will be ergodic in the infinite time limit, 
    \item An approach to generate finite-time ergodic trajectories in the presence of obstacles using measure-preserving vector fields, 
    \item An amendment to the ergodic metric to handle complex environments and generate collision-free trajectories that minimize the ergodic metric, and 
    \item A generalization to multi-agent systems to generate joint trajectories that naturally avoid inter-agent collisions. 
\end{enumerate}

The rest of the paper is organized as follows. \cref{sec:litreview} summarizes prior ergodic search techniques that deal with obstacle and collision avoidance. \cref{sec:theory} introduces the theory and proofs justifying the use of measure-preserving vector fields as well as the proposed Laplace-Beltrami eigenfunctions. \cref{sec:results} shows experimental results of simulated search trajectories on test maps, and comparisons against other methods. Lastly, \cref{sec:conclusion} concludes the paper and poses a few future work options.

\section{Literature Review} \label{sec:litreview}

\subsection{Ergodic search amidst obstacles}
\label{sec:litreview_ergodic}

In the context of non-uniform information maps, ergodic search-based methods have proven to be quite effective in exploring a space\cite{patel2021multi,choset2001review,murphey2013trajopt}. They showcase an ability to balance exploration and exploitation, where the former attempts to visit all possible locations for new information while the latter searches high-information regions myopically \cite{murphysearch,mezic2011lawnmover,ren2023MOES}. However, historically, ergodic search in the presence of obstacles has been a challenge. 

One of the first methods that tackle obstacles places repulsive vector fields around idealized circular obstacles \cite{salman2017multi}. The repulsive field strength is scaled with the distance from the robot to the obstacle, guaranteeing that the agent will never collide with obstacles as long as the integration time step is sufficiently fine. However, the repulsive fields are only defined for circular obstacles, making it difficult to conduct a comparison study against other methods for more complex environments. 

Alternatively, Ayvali {\em et. al.} \cite{ayvali2017ergodic} use stochastic trajectory optimization for ergodic coverage (STOEC) to compute coverage trajectories, which lends itself naturally to obstacle avoidance. During the stochastic sampling step, rejecting all the samples that lead to an obstacle intersection is an easy way to guarantee that the chosen trajectory will be collision-free. However, the drawback of this method is implicit. To ensure good ergodic coverage, the duration of sampled trajectories must be long enough. However, the probability of randomly sampling a valid trajectory in a convoluted environment vanishes exponentially as the time horizon increases. In practice, STOEC struggles to find paths around large obstacles or maze-like environments. 

Lerch et al. \cite{lerch2023safety} propose a control barrier function to prevent the trajectory from entering the obstacle volumes, then use gradient-based optimization methods to numerically minimize the ergodic metric with respect to the trajectory and control inputs. Unfortunately, this method also struggles with convoluted maps since the computed trajectories get stuck behind large obstacles or inside small rooms.

\subsection{Laplace-Beltrami basis functions for non-Euclidean spaces}
\label{sec:litreview_laplace}
The Fourier basis functions are commonly used for spectral analysis in $\mathbb R^n$. They appear in the ergodic metric to help us compare the spectral coefficients of the information distribution and the time-averaged robot statistics. However, it is well known that they are not appropriate when the underlying space is not square \cite{byerly1893elemenatary,macrobert1967spherical}, as evidenced by the necessity of Laplace's spherical harmonic functions. The most general solution for spectral analysis on an arbitrary manifold uses the eigenfunctions of the Laplace-Beltrami operator. It is used in computational geometry to summarize functions on meshes \cite{levy2006laplace,solomon2014laplace}, especially when the space has curvature or complicated boundaries. In a similar vein, we utilize it to summarize the spatial and temporal statistics of our information distribution in ergodic search, especially in the presence of complex obstacles. 

\section{Theory} \label{sec:theory}

We begin our discussion on ergodic theory with a compact metric space $X$, which for this paper we assume is a subset of $\mathbb R^n$ equipped with the Borel $\sigma$-algebra $\mathscr B$. In addition, we have some known information distribution represented by a probability measure $\mu:\mathscr B\to \mathbb R$. Together these objects form a measure space $(X,\mathscr B,\mu)$.
We study the behavior of a robot on this space by treating the robot as a point $x\in X$ and its time evolution as a function $T^t: (X,\mathbb R)\to X$; the robot's configuration space is $X$ and a robot starting at $x$ will end up at $T^t(x)$ after $t$ seconds. Thus the ergodic criterion is the convergence of the robot's trajectory statistics to the desired distribution, stated in (\ref{eq:erg-criterion}).
\begin{equation} \label{eq:erg-criterion}
    \frac{1}{\tau}\int_0^\tau \delta_{[T^t(x)]}dt \to \mu
\end{equation}

In this section, we detail the generation of ergodic trajectories by first studying measure preserving flows, an essential component of ergodic systems. Then we prove that following the flows results in an ergodic trajectory by showing that they are uniquely ergodic. Lastly, we discuss ergodic search on a fixed time horizon, using the randomly generated trajectory as an initial guess and minimizing the ergodic metric using gradient-based methods. 

\subsection{Measure-preserving flows}
Fundamental to ergodic theory is the concept of \textit{measure-preserving flows}, as measure-preserving is a prerequisite for ergodicity \cite{einsiedler2011ergodic}. Thus it is useful to search for an ergodic flow within the space of all measure-preserving flows. Flows (\cref{defn:flow}) are functions that simulate motion along a vector field, whether that field is time varying or not. A flow is called measure-preserving (\cref{defn:measure-preserving}) with respect to a measure space $(X, \mathscr B, \mu)$ if it leaves the measure unchanged.

\begin{definition}[Flow] \label{defn:flow}
    We define a \textbf{time-invariant flow} as a family of continuous functions $T^t:X\to X$ such that for any $s,t\in\mathbb R$, $T^t(T^s(x)) = T^{t+s}(x)$. Each flow is related to a vector field over $X$ by the exponential map; $T=\exp(v)$ and $v=\log(T)$. Equivalently, the time-invariant flow can be defined as the flow along its vector field $v(x)$.

    On the other hand, a \textbf{time-varying flow} is defined as the flow along a time-varying vector field $v(x,t)$. 
\end{definition}

We note that our time parameter $t\in\mathbb R$ is a real number, thus admitting a negative time flow. For this paper, we assume our flows are invertible, and thus $T^{-t}$ is the inverse map of $T^t$.

\begin{definition}[Measure-preserving]\label{defn:measure-preserving}
    Given a measure space $(X,\mathscr B,\mu)$ and a flow $T^t:X\to X$, $T$ is \textbf{measure-preserving} on $\mu$ if for all subsets $A\subseteq X$, $\mu(A)=\mu(T^{-t}A)$.

    The converse is referred to as \textbf{$T$-invariance}, where the same property indicates $\mu$ is $T$-invariant.
\end{definition}

To describe the space of measure-preserving flows, we can look at the differential change in measure that would result from moving a point in $X$ along a vector field $v$. It turns out the space of all measure-preserving flows can be identified with the solutions to a partial differential equation (\cref{thm:measure-preserving-divergence}).
\begin{theorem}[Measure-preserving flows]\label{thm:measure-preserving-divergence}
    Let $(X,\mathscr B,\mu)$ be a measure space (we assume $X\subseteq \mathbb R^n$). Let $p(x)$ be a continuous and differentiable probability density function associated with $\mu$ such that $\mu(A)=\int_A p(x)dx$.

    The set of all measure-preserving flows on this measure space can be described through their associated vector fields $T = \exp(\vec v)$. Further, these vector fields form a linear subspace described by the following partial differential equation.
    \begin{equation} \label{eq:vf-solutions}
      \nabla\cdot (p(x) \vec{v}(x)) = 0
    \end{equation}
\end{theorem}
\begin{proof}
    Proof in \cref{app:pf-thm1}
\end{proof}

As the solutions to (\ref{eq:vf-solutions}) form a linear subspace, we can approximate that subspace to a certain extent using a finite basis of vector fields, all of which satisfy the measure-preserving criterion. Our paper focuses primarily on an application to a 2D setting, and one method for finding measure-preserving vector fields on 2D spaces is to use the 2D curl. First, we choose a set of basis functions $\{u_i(x)\}$ as the eigenfunctions of the Laplacian, subject to a Dirichlet boundary condition $u(x) = 0$ for all $x\in\partial X$.
\begin{equation}
    \nabla^2 u_i = \lambda_i u_i
\end{equation}

Then we can construct vector fields using the 2D curl, where $x$ and $y$ denote the two dimensions. We can see that the vector fields in equation (\ref{eq:measure-preserving-laplacian-eigfuncs}) always satisfy (\ref{eq:vf-solutions}), and on the boundary $\partial X$, $v$ will always point parallel to the boundary.
\begin{equation}\label{eq:measure-preserving-laplacian-eigfuncs}
    v_i(x, y) = \left[-\frac{1}{p(x, y)} \frac{\partial u_i(x, y)}{\partial y}, \frac{1}{p(x,y)}\frac{\partial u_i(x, y)}{\partial x}\right]
\end{equation}

We note that following these basis flows, or any linear combination of them, will result in limit cycles and will \textit{not} yield ergodic trajectories. This effect is visualized in \cref{fig:measure-preserving-flows-example} with the traced trajectory (solid and faded) for each respective flow field. Measure-preserving is a necessary condition for ergodicity, but is not sufficient. The process for generating ergodic trajectories using these measure-preserving flows requires time variance and is detailed in the next section.

\begin{figure}[t!]
    \centering
    \includegraphics[width=\textwidth]{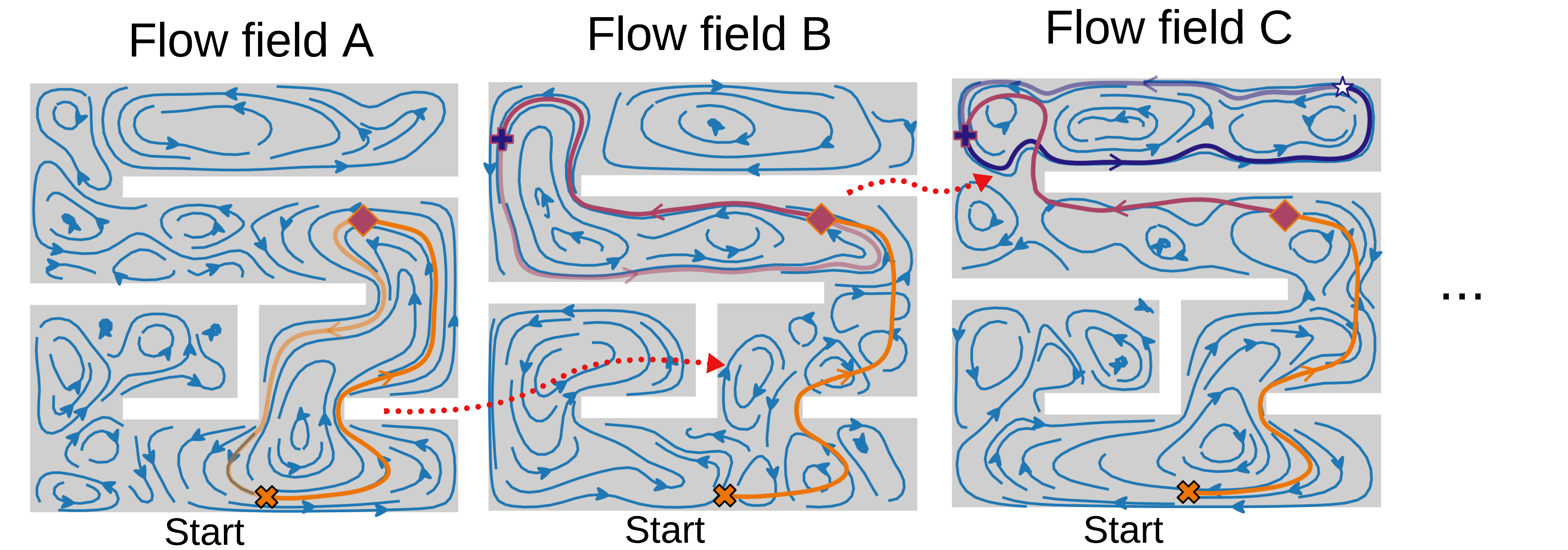}
    \caption{Example measure-preserving flows for a uniform information distribution on a maze-like space, drawn in blue. (A) Starting from an initial condition marked by the brown X, the brown line represents a  trajectory traced by following the flow field lines.
    Note that the flow lines form stationary limit cycles, as shown by the faded brown extension, so the robot will simply trace the same cycle repeatedly.
    (B) By cleverly switching the underlying flow field when the robot reaches the magenta diamond, we can escape the limit cycle and cover more space.
    (C) We can repeat this process when the robot reaches the purple plus, extending the trajectory to cover more of the map. This process can be repeated again when the robot reaches the white star ad infinitum.
    }
    \label{fig:measure-preserving-flows-example}
    \vspace{-1.5em}
\end{figure}

\subsection{Infinite-Limit Ergodicity from Random Flows}
As mentioned in the previous section, tracing a 2D time-invariant flow will always result in a stationary limit cycle. To break out from the limit cycle, we need to follow a time-varying flow. This concept is illustrated in \cref{fig:measure-preserving-flows-example}, where switching between flow fields A, B, and C at key points helps extend the trajectory to cover more of the space.

A random flow (\cref{defn:random-flow}) is a simple method to generate a time-varying flow. In particular the random flow $\mathfrak T^t$ is a Markov process, and by the Krylov-Bogolyubov theorem \cite{kryloff1937theorie} admits at least one invariant measure.

\begin{definition}[Random flow]\label{defn:random-flow}
    Let $\mathfrak T^t$ be a time-invariant distribution of flows such that if $T^t\sim \mathfrak T^t, T^s\sim \mathfrak T^s$, then $T^t(T^s(\cdot))\sim \mathfrak T^{t+s}$. $\mathfrak T^t$ is a Markov process.
    Note that the sampled flows are time-varying, but the overall stochastic process is time-invariant.

    Invariant measures are defined in terms of $\mathfrak T^t$'s push-forward measure $(\mathfrak T^t)_*$. $\mu$ is invariant for $\mathfrak T^t$ if $(\mathfrak T^t)_*\mu = \mu$.
\end{definition}

Notably, if $\mathfrak T^t$ admits only one invariant measure, we can invoke \textit{unique ergodicity} (\cref{thm:random-unique-ergodicity}) to prove that $\mathfrak T^t$ is ergodic. In other words, if there is exactly one measure $\mu$ invariant with respect to $\mathfrak T^t$, then $\mathfrak T^t$ must be ergodic with respect to $\mu$. And therefore the spatial statistics of the trajectory traced by following $T^t$ will converge to $\mu$ as $t\to\infty$.

\begin{theorem}[Unique Ergodicity] \label{thm:random-unique-ergodicity}
    Let $(X,\mathscr B,\mu)$ be a measure space and let $\mathfrak T^t$ be a distribution of random flows. If $\mu$ is the only measure stationary with respect to $\mathfrak T^t$, then we say $\mathfrak T^t$ is uniquely ergodic. And in the infinite limit, trajectories sampled from $\mathfrak T^t$ will be ergodic.
\end{theorem}
\begin{proof}
    See Eisiedler's textbook \cite{einsiedler2011ergodic}, Chapter 4.3. 
\end{proof}

Using the basis vector fields $\{v_1,\cdots,v_n\}$ from (\ref{eq:measure-preserving-laplacian-eigfuncs}), we can construct a distribution uniquely ergodic with respect to the given measure. With $U_1\cdots U_n$ being $n$ uniform random variables on the range $[-1, 1]$ and $u_i\sim U_i$, let $\mathfrak T^t$ be the random flow obtained by exponentiating the random vector field $V^t$. 
\begin{equation} \label{eq:random-vector-field-laplacian-eigfunc}
    \sum_{i=1}^n u_i v_i \sim V^t
\end{equation}
\begin{corollary}[Constructing a Uniquely Ergodic Random Flow] \label{coro:random-unique-erg-laplacian-eigfunc}
    Let $\mathfrak T^t$ be the random flow obtained from exponentiating (\ref{eq:random-vector-field-laplacian-eigfunc}). If the support of the desired measure $\mu$ has only one connected component, then $\mathfrak T^t$ is ergodic.
\end{corollary}
\begin{proof}
    Any measure invariant with $V^t$ must be invariant for every basis vector field $\{v_i\}$. By construction, $\mu$ is invariant for all the $\{v_i\}$ and invariant with $V^t$.

    On the other hand, there is no other measure $\nu$ invariant with every $\{v_i\}$, unless $\mu$ can be split (i.e. across multiple connected components). By contradiction, suppose $\nu\neq\mu$ is invariant with all $\{v_i\}$. 
    The boundary of $\nu$'s support must be a shared limit cycle of all $\{v_i\}$. 
    But if the support of $\mu$ has only one connected component, then the support of $\nu$ would have to match the support of $\mu$ and thus $\nu=\mu$, leading to a contradiction.

    Thus $\mu$ is the only invariant measure, and $\mathfrak T^t$ is ergodic by \cref{thm:random-unique-ergodicity}. 
    \hfill $\qed$
\end{proof}

\subsection{Ergodic metric on arbitrary 2D metric space}
\begin{figure}[t!]
    \centering
    \includegraphics[width=.9\textwidth]{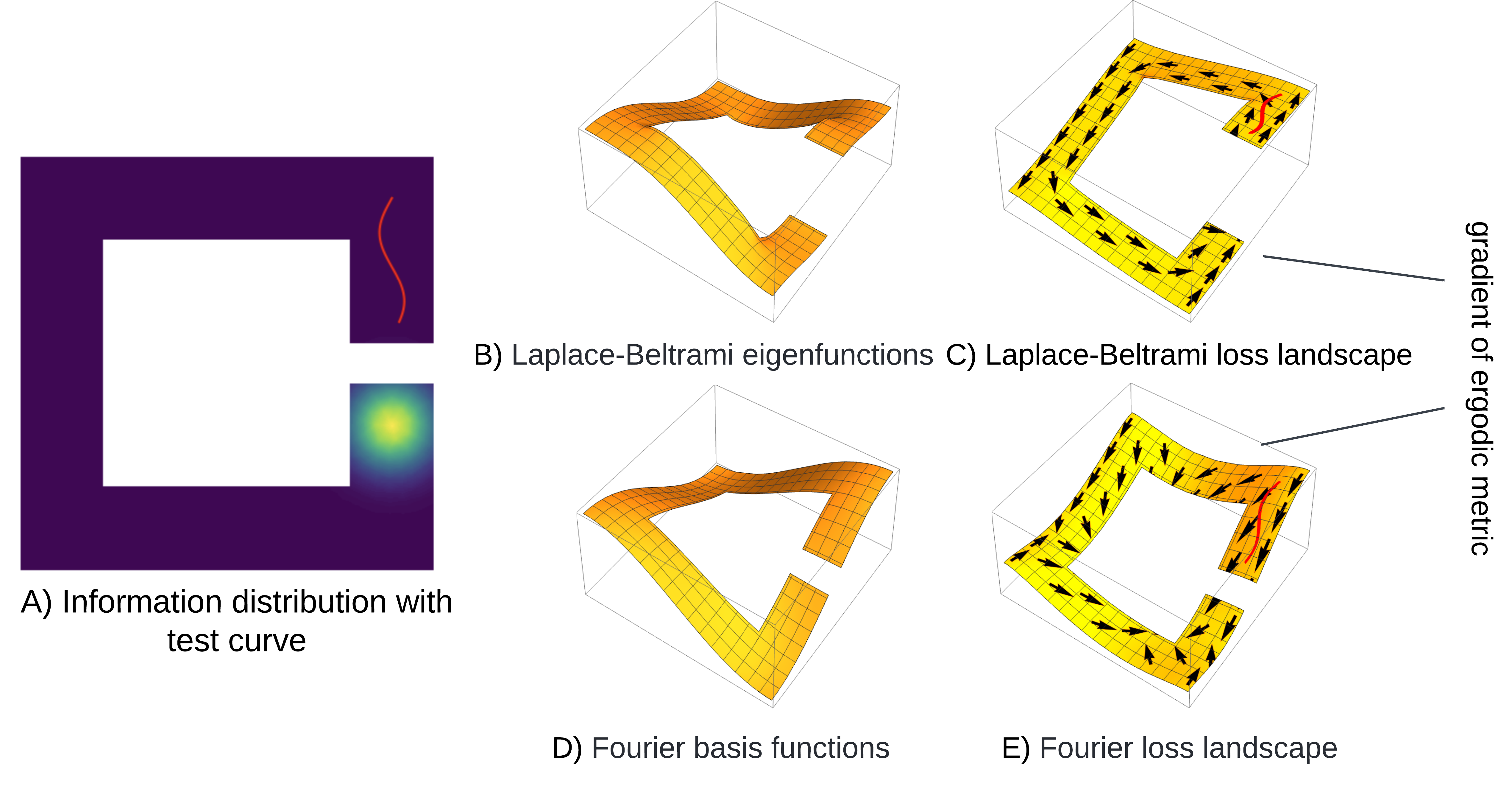}
    \caption{Effect of modifying the ergodic metric to use the Laplace-Beltrami eigenfunctions. (A) A C-shaped map with an information peak (yellow) and test curve (red) are separated by a gap.
    Both basis function choices depend only on the geometry of the obstacle map and can be compared in figures B and D. The Laplace-Beltrami eigenfunctions (B) have the correct sense of continuity influenced by the boundaries of the free space, while the ordinary Fourier basis functions (D) try to maintain continuity within the gap.
    As a result, the Laplace-Beltrami loss landscape (C) smoothly pushes the test curve around the central obstacle, while the Fourier loss landscape (E) gets trapped in a local minima.
    }
    \label{fig:laplace-beltrami-metric}
    \vspace{-1.5em}
\end{figure}

According to \cref{coro:random-unique-erg-laplacian-eigfunc}, randomly sampling a trajectory from $\mathfrak T^t$ as time $t\to\infty$ guarantees an ergodic trajectory, covering our space according to the given information distribution. However, there is no mention of how quickly the convergence happens or how well the trajectory covers the space for a finite time horizon.

Thus we turn to the ergodic metric \cite{mathew2011metrics} used in many other ergodic search formulations to maximize coverage in a finite time horizon. However, rather than the standard Fourier basis functions for a square map $\phi_{ij}=\cos(k_i\pi x)\cos(k_j\pi y)$, we must instead find the spatial harmonic basis functions for our arbitrary metric space $X$ using the eigenfunctions of the Laplace-Beltrami operator \cite{graichen2015sphara} (\ref{eq:laplace-beltrami}).
\begin{equation} \label{eq:laplace-beltrami}
    \nabla^2 \phi_k = \lambda_k\phi_k
\end{equation}

Because our space is digitally represented as a triangular mesh, we use the Finite Element Method (FEM) to compute mesh basis functions on the graph Laplace-Beltrami operator \cite{vallet2008spectral,dyer2007investigation,zhang2007spectral}. This yields the basis functions $\phi_k$, with integrals approximated as the inner product $\int_X fg dx \approx f^\top M g = \langle f,g\rangle_M$. The $k$th spectral coefficient of the information distribution $p(x)$ is $\hat\xi_k = \langle p, \phi_k\rangle_M$
and the $k$th spectral coefficient of the agent trajectory is given as
\begin{equation}
    \xi_k = \frac{1}{T}\sum_{t=1}^T \langle \phi_k, \delta_{x_t} \rangle_M= \frac{1}{T}\sum_{t=1}^T \phi_k(x_t)
\end{equation}

The ergodic metric itself is the squared error between the information distribution and trajectory's spectral coefficients, weighted to place more importance on the low-frequency coefficients (\ref{eq:ergodic-metric}).
The weighting formula $(1+\sqrt{\lambda_k})^{-2}$ is chosen to ensure consistent behavior on an obstacle-free square map when compared to prior literature. We will henceforth refer to our modified ergodic metric as the Laplace-Beltrami ergodic metric to distinguish it from the previous Fourier ergodic metric.
\begin{gather} \label{eq:ergodic-metric}
    \mathcal E = \sum_{k=0}^\infty \left(1+\sqrt{\lambda_k}\right)^{-2} \left(\xi_k - \hat\xi_k\right)^2
\end{gather}

\subsection{Ergodic metric minimization for finite time horizon}
To minimize the ergodic metric (\ref{eq:ergodic-metric}) on a finite time horizon, we formulate the problem as a minimization problem over the state $x_a^t$ for every agent at every time and vector field coefficients $u_i(t)$. The first constraint comes from forcing the trajectory to follow the measure-preserving flows (\ref{eq:random-vector-field-laplacian-eigfunc}), and the second constraint enforces a maximum velocity.

\begin{equation} \label{eq:erg-minimization-problem}
\begin{aligned}
\min_{x_a^t,u_i(t)} \quad & \mathcal E(x_a^t) \\
\textrm{s.t.} \quad & \dot{x}_a^t = \sum_i u_i(t) v_i(x_a^t) \quad&\forall ~t\in [0,T] \\
& \| \dot{x}_a^t \| \leq v_\text{max} &\forall ~t\in [0,T]
\end{aligned}
\end{equation}

Picking some time discretization $\Delta t$ and writing $x^{t+1}_a = f(x^t_a,u_i(t))$ as the flow function for following $\dot x^t_a$ for one timestep, we can rewrite the constraints of (\ref{eq:erg-minimization-problem}) as $x_a^{t+1} = f(x_a^t, u_i(t))$ and $\|x_a^{t+1} - x_a^{t}\|\leq v_\text{max}\Delta t$.

We use gradient descent to solve this problem, forcing the $\Delta x$ and $\Delta u$ updates to respect the linearized constraints at every iteration.
$\eta$ is gradient step size and $\kappa$ is a penalty term for the change in vector field.

\begin{equation} \label{opt:grad-calc}
\begin{aligned}
\min_{\Delta x,\Delta u} \quad & \|\eta\nabla_x \mathcal E - \Delta x\|^2 + \kappa \|\Delta u\|^2 \\
\textrm{s.t.} \quad & x_a^{t+1} + \Delta x_a^{t+1} = f(x_a^t, u_i(t)) + \partial_x f(\cdot) \Delta x_a^t + \partial_u f(\cdot) \Delta u  \\
  &\|x_a^{t+1} + \Delta x_a^{t+1} - x_a^{t} - \Delta x_a^{t}\| \leq v_\text{max}\Delta t    \\
\end{aligned}
\end{equation}

With the linear approximation of the flow constraints, we notice that (\ref{opt:grad-calc}) is a second-order cone problem (SOCP), and we minimize ergodicity iteratively by running a SOCP solver to find the $\Delta x$ and $\Delta u_i$ updates.

\section{Results and Discussion} \label{sec:results}
\begin{wrapfigure}{r}{0.5\columnwidth}
    \centering
    \vspace{-5em}
    \includegraphics[width=\linewidth]{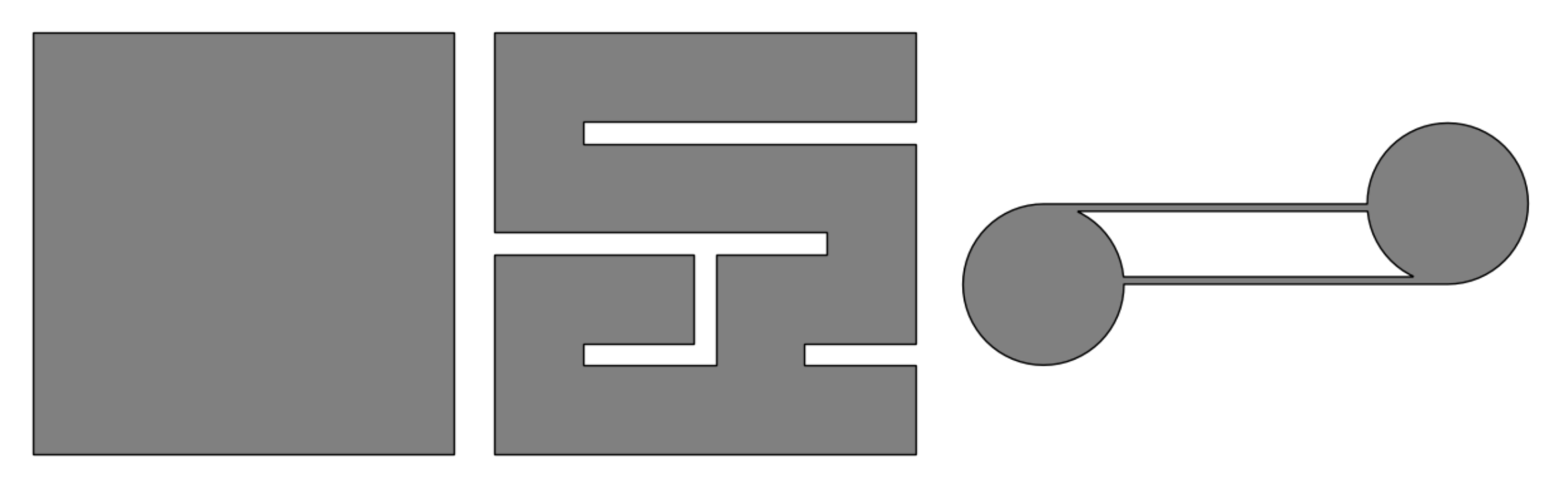}
    \caption{Obstacle maps used for testing. From left to right, we have the obstacle-free map, the Maze map, and the Rooms map.}
    \label{fig:maps}
    \vspace{-1.5em}
\end{wrapfigure}
In order to evaluate the efficacy of our approach in complex environments, we present results from experiments on three maps (see \cref{fig:maps}) with uniform and nonuniform information distributions. We also compare the performance of various algorithms in single and multi-agent settings. In particular, we compare our results against STOEC \cite{ayvali2017ergodic} and the control barrier function method \cite{lerch2023safety}, observing at least an order of magnitude improvement in performance as the obstacles in the environment get more complex. 
Note that the experiments are run with a max velocity of 1 unit/s, with the single-agent case running for a duration of $t=5$s. The multi-agent comparison cases use a total of 7 agents, with each agent having a trajectory of $t=1$s.



\subsection{Obstacle-free Map Experiment Results}
The obstacle-free map is used as a baseline to ensure that our method matches the performance of control barrier function-based implementation \cite{lerch2023safety}. 
The experiment results for the single-agent uniform (Left), single-agent nonuniform (Center), and multi-agent nonuniform (Right) cases are shown in \cref{fig:square-exp}. 
The final ergodic metric scores for the four experiments, i.e., uniform or non-uniform (gaussian) coverage using single or multiple agents, are reported in \cref{tab:square-results}, with \textit{F} being the Fourier ergodic metric and \textit{LB} being our modified Laplace-Beltrami ergodic metric. We observe that our method performs comparably to the baseline. Note that for the square map, the Fourier and Laplace-Beltrami ergodic metrics are identical. The numerical difference in the table is solely due to approximation error.
\begin{figure}[h!]
    \centering
    \vspace{-2em}
    \includegraphics[width=.32\textwidth]{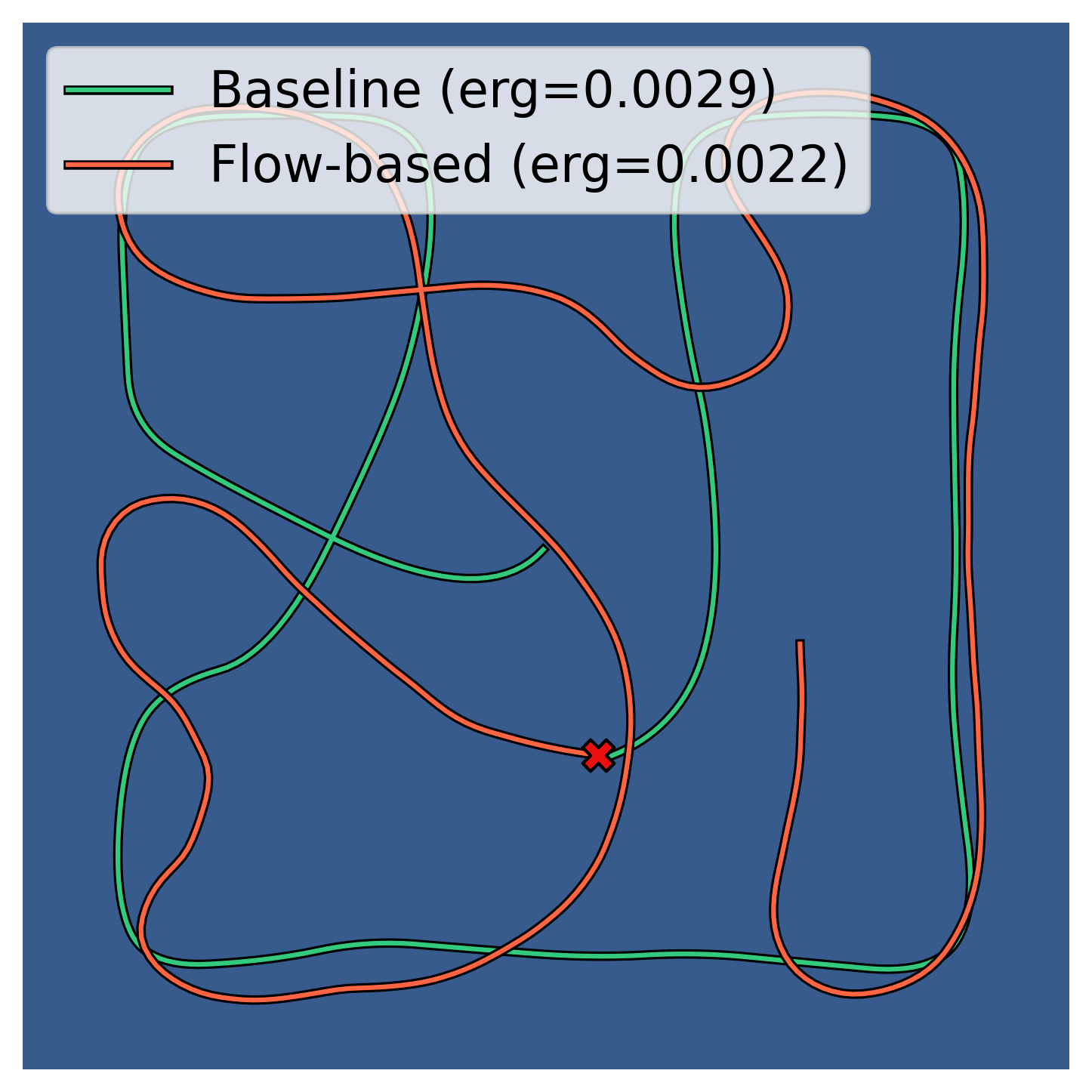}
    \includegraphics[width=.32\textwidth]{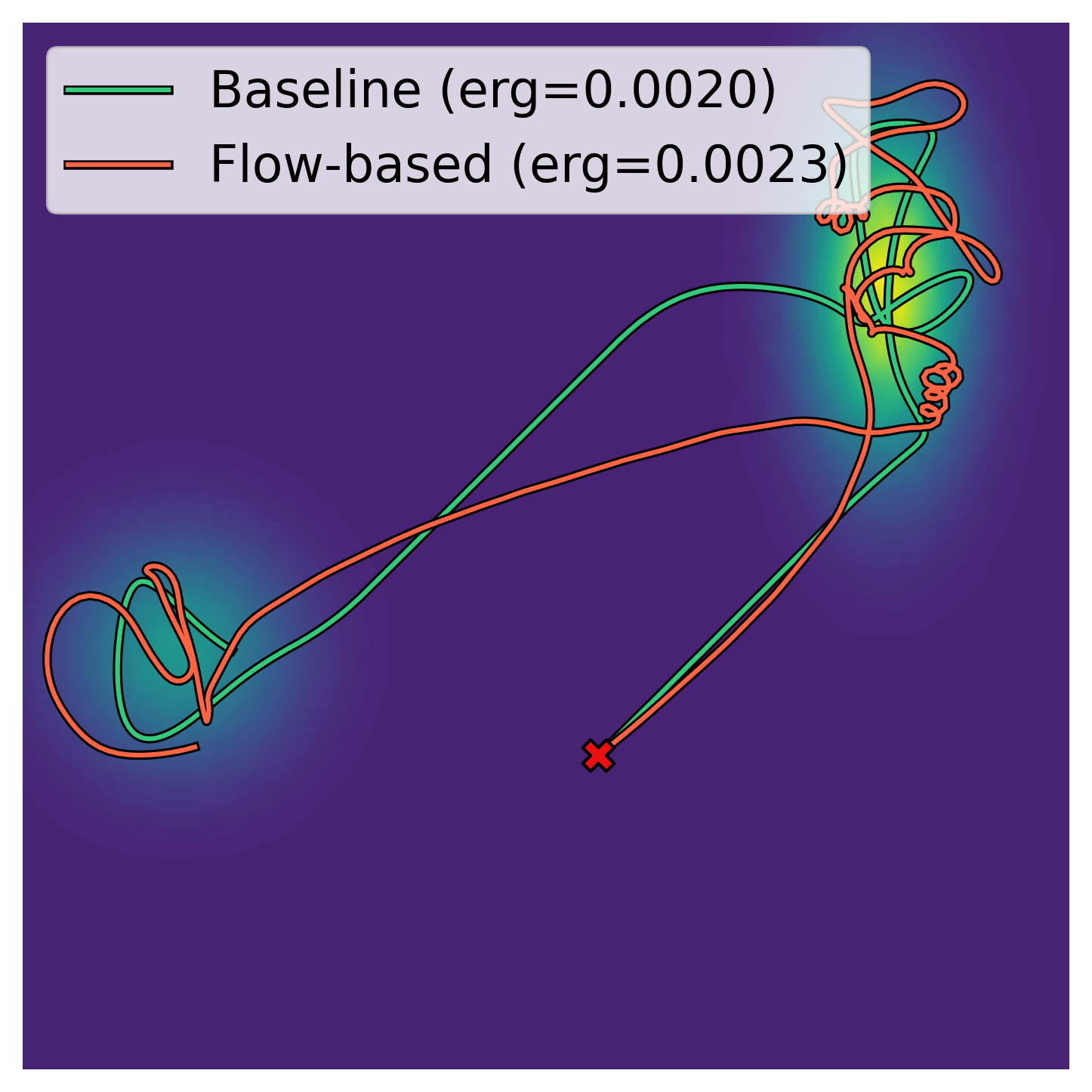}
    \includegraphics[width=.32\textwidth]{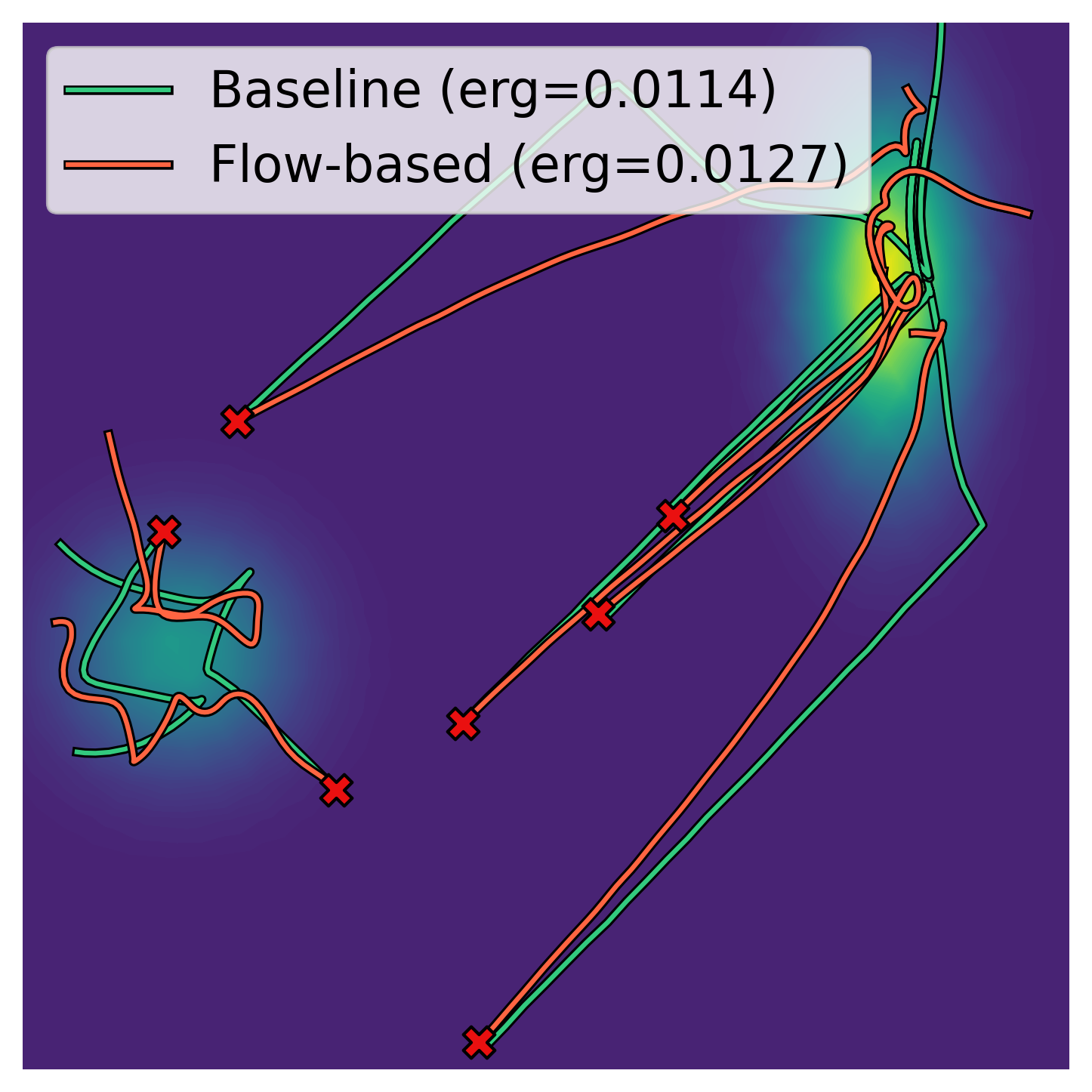}
    \caption{Results for obstacle-free square map comparing our and the baseline methods. From left to right, we have the trajectories generated for the uniform, nonuniform, and multi-agent nonuniform test cases. The ergodic metric value for each trajectory is shown in the figure legend.
    }
    \label{fig:square-exp}
    \vspace{-3em}
\end{figure}

\begin{table}[h!]
\centering
\vspace{-1em}
\caption{
The Fourier (F) and Laplace-Beltrami (LB) ergodic metric values for the trajectory produced by each method on the obstacle-free square map. The ergodicity values between these approaches are comparable in the absence of obstacles. 
}
\label{tab:square-results}
\begin{tabular}{|c|cccc|cccc|}
\hline
                  & \multicolumn{4}{c|}{Single Agent}                            & \multicolumn{4}{c|}{Multi-Agent}                             \\
Obstacle-free map & \multicolumn{2}{c}{Uniform} & \multicolumn{2}{c|}{Gaussians} & \multicolumn{2}{c}{Uniform} & \multicolumn{2}{c|}{Gaussians} \\
                  & F  & LB                     & F             & LB             & F  & LB                     & F             & LB             \\ \hline
Baseline \cite{lerch2023safety} & 0.0028 & \multicolumn{1}{c|}{0.0029} & 0.0019 & 0.0020 & 0.0006 & \multicolumn{1}{c|}{0.0005}  & 0.0118 & 0.0114 \\
Flow-based (ours) & 0.0022 & \multicolumn{1}{c|}{0.0022} & 0.0025 & 0.0023 & 0.0004 & \multicolumn{1}{c|}{0.0004}  & 0.0131 & 0.0127 \\ \hline
\end{tabular}
\vspace{-3em}
\end{table}

\subsection{Maze Map Experiment Results}
The maze map poses a path-finding problem in addition to the search problem, and the challenge lies in the large distance between points on opposite sides of a wall and the back-and-forth snaking required to navigate the maze. The trajectories computed from various approaches are illustrated in \cref{fig:maze-exp} and their respective ergodicity values are summarized in \cref{tab:maze-results}. The maze map results demonstrate that our flow-based method is not only able to avoid the obstacles, but can also navigate around large obstacles. 
We observe that although STOEC\cite{ayvali2017ergodic} also successfully avoids wall collisions, it gets trapped in some corner of the maze, unable to explore the entire map. The control barrier function approach\cite{lerch2023safety}, on the other hand, hugs the walls very closely and occasionally navigates through the wall leading to infeasible trajectories.
In summary, we observe that our method is able to better explore the maze in all cases, evidenced by the lower ergodic metric value.

\begin{figure}[t!]
    \centering
    \vspace{-1.5em}
    \includegraphics[width=.29\textwidth]{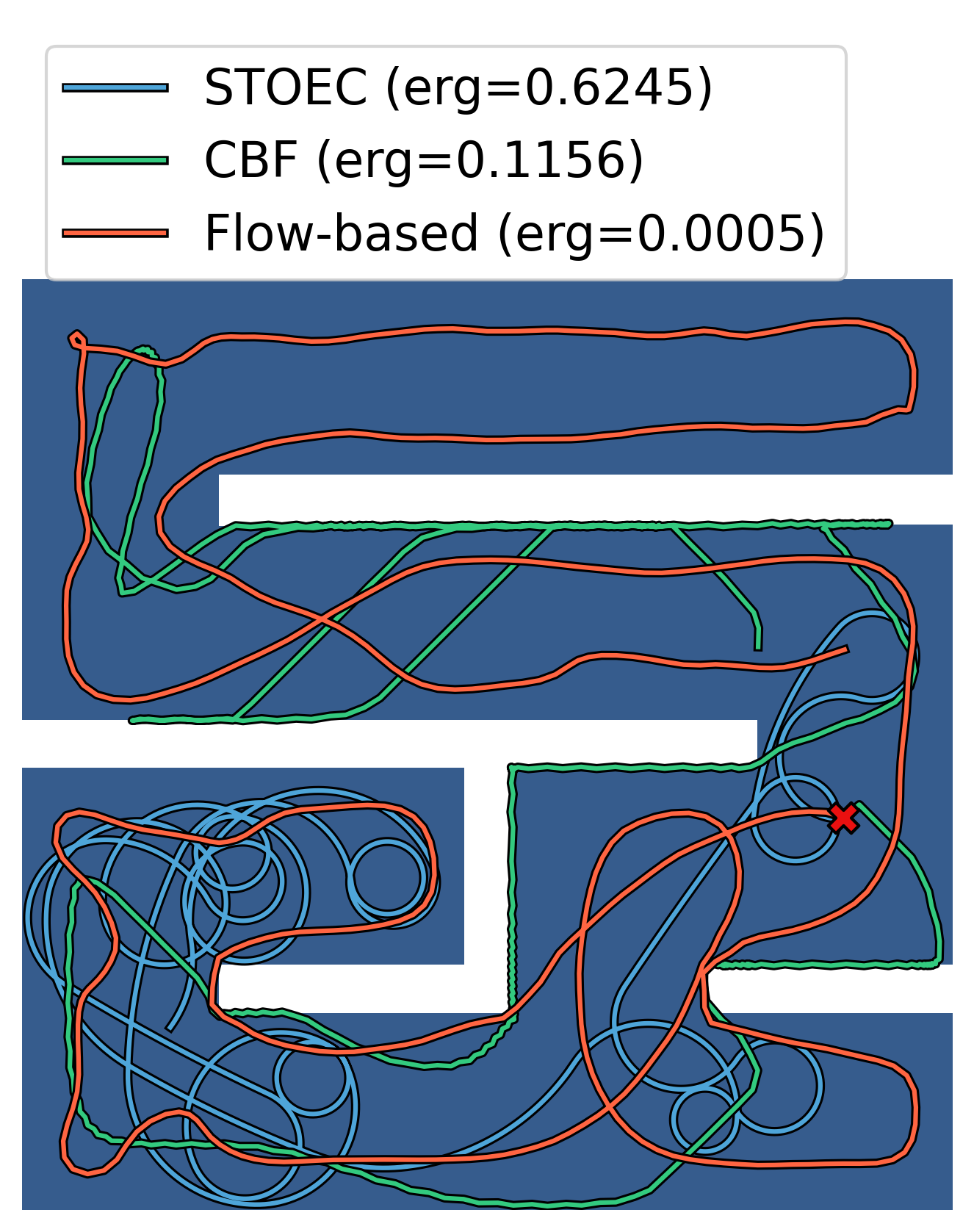}
    \includegraphics[width=.29\textwidth]{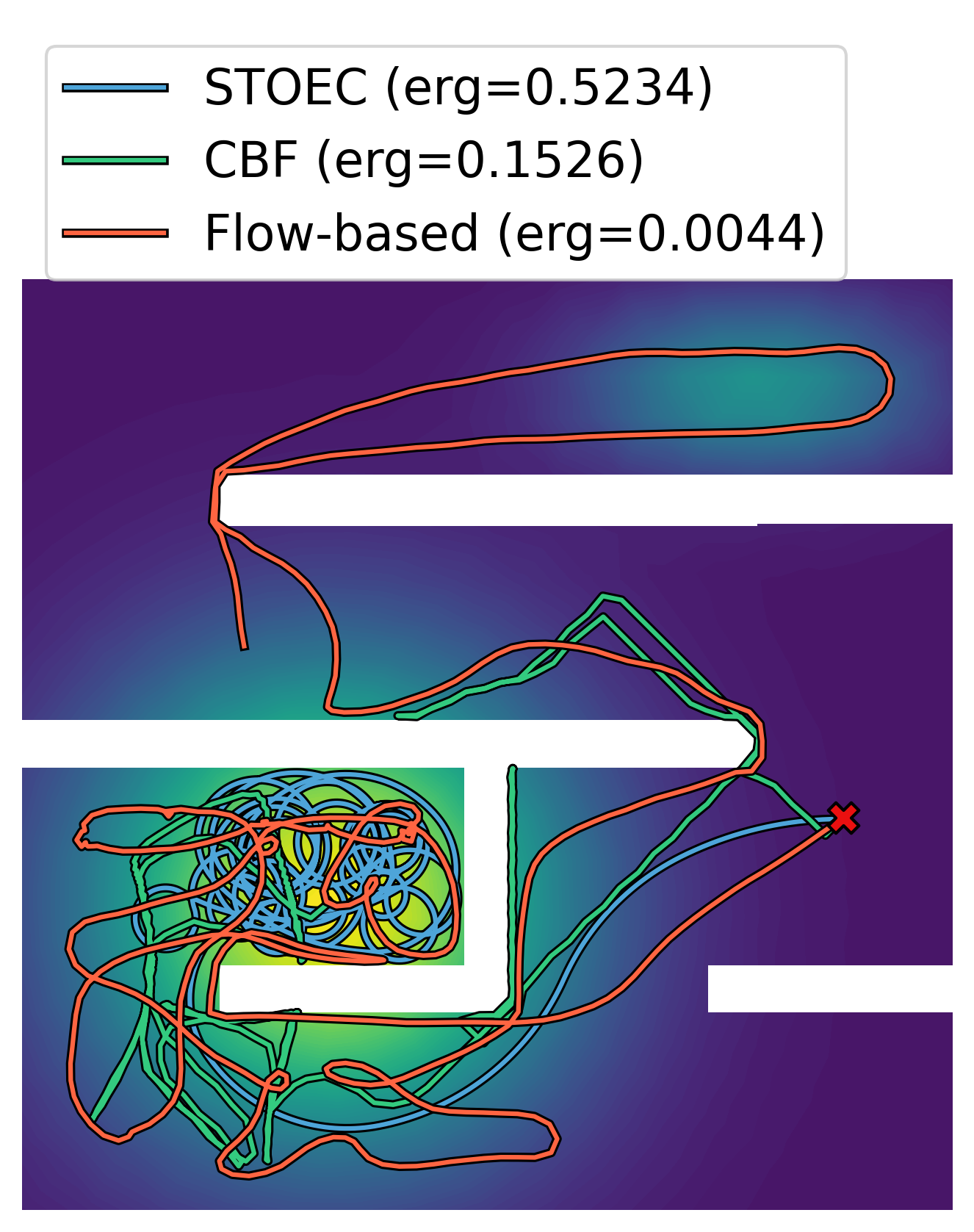}
    \includegraphics[width=.29\textwidth]{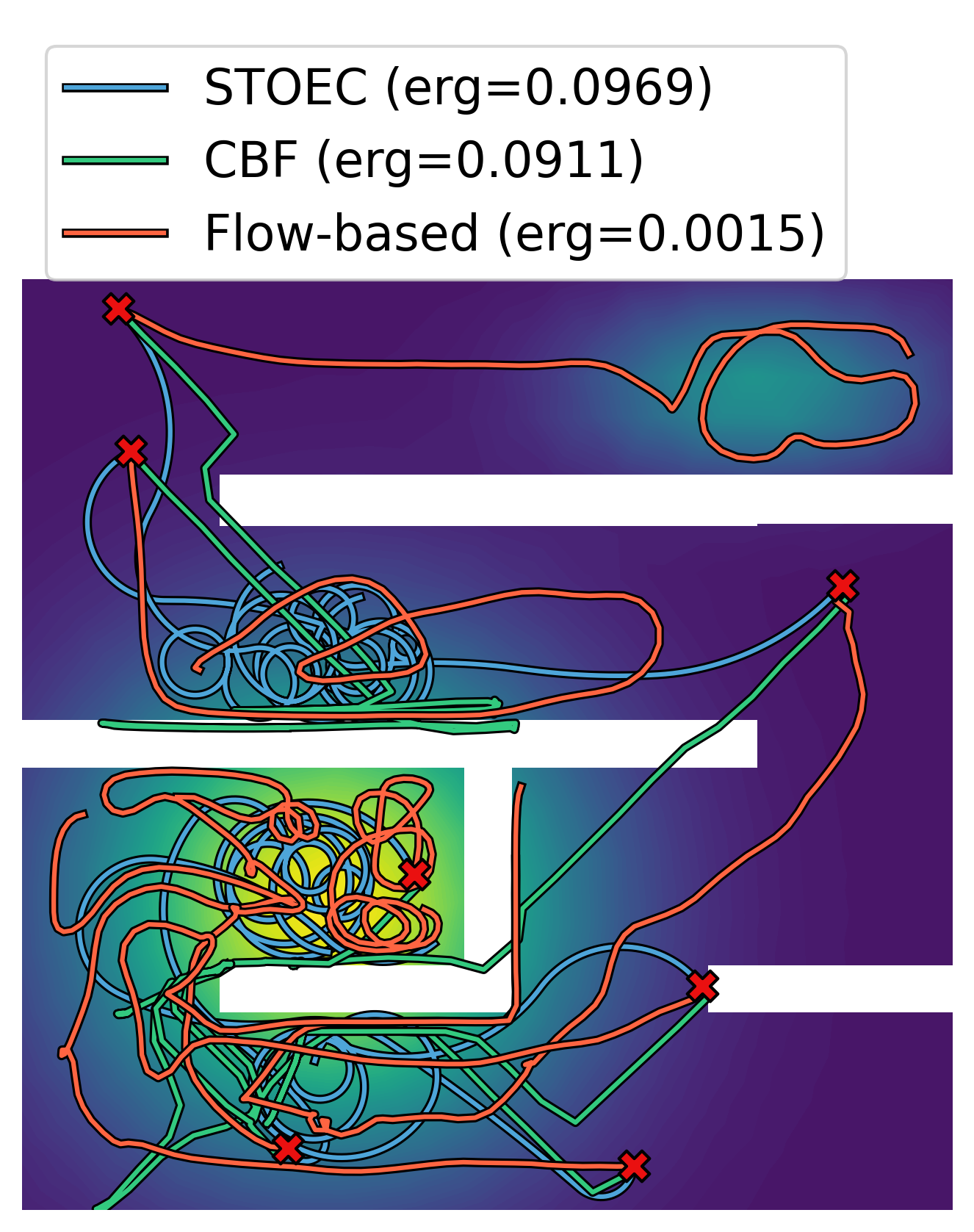}
    \caption{Results for maze map comparing our and prior methods. 
    From left to right, we have the trajectories generated for the uniform, nonuniform, and multi-agent nonuniform test cases.
    The ergodic metric value for each trajectory is shown in the figure legend, calculated using the Laplace-Beltrami ergodic metric.}
    \label{fig:maze-exp}
    \vspace{-1.5em}
\end{figure}

\begin{table}[h!]
\centering
\vspace{-1.5em}
\caption{
The Fourier (F) and Laplace-Beltrami (LB) ergodic metric values for the trajectory produced by each method on the maze map. When compared with two prior works, our approach shows at least an order of magnitude improvement in the ergodicity values.
}
\label{tab:maze-results}
\begin{tabular}{|c|cccc|cccc|} \hline
    & \multicolumn{4}{c|}{Single Agent} & \multicolumn{4}{c|}{Multi-Agent} \\
    Maze map& \multicolumn{2}{c}{Uniform} & \multicolumn{2}{c|}{Gaussians} & \multicolumn{2}{c}{Uniform} & \multicolumn{2}{c|}{Gaussians}  \\
    & F&LB & F&LB & F&LB & F&LB \\\hline
    STOEC \cite{ayvali2017ergodic} & 0.0658 & \multicolumn{1}{c|}{0.6245} & 0.0480 & 0.5234 & 0.0095 & \multicolumn{1}{c|}{0.0552}  & 0.0194 & 0.0976 \\
    Control barrier \cite{lerch2023safety} & 0.0067 & \multicolumn{1}{c|}{0.0710} & 0.0392 & 0.1610 & 0.0011 & \multicolumn{1}{c|}{0.0160}  & \multicolumn{2}{c|}{{\color[HTML]{FF0000} Collision}} \\
    \rowcolor[HTML]{9AFF99}  Flow-based (ours) & 0.0004 & \multicolumn{1}{c|}{0.0005} & 0.0022 & 0.0044 & 0.0001 & \multicolumn{1}{c|}{0.0001}  & 0.0007 & 0.0015 \\
    \hline
\end{tabular}
\vspace{-3em}
\end{table}

\subsection{Rooms Map Experiment Results}
Lastly, the rooms map connects two circular rooms with two narrow corridors. Robots must move in a straight, single-file line through the corridors to avoid collision.
The trajectories computed from various approaches are illustrated in \cref{fig:rooms-exp} and their respective ergodicity values are summarized in \cref{tab:rooms-results}. The rooms map results show that our method is able to traverse the long corridors and explore the full area.
We observe that STOEC\cite{ayvali2017ergodic} is unable to cross from the left room to the right, as the corridors are too narrow. And while the control barrier approach\cite{lerch2023safety} is able to cross over, it often gets stuck with parts of the trajectory in forbidden regions. In summary, our method not only explores both rooms successfully in all the test cases but also minimizes the ergodic metric, as evidenced by the lower ergodicity values.

\begin{figure}[t!]
    \centering
    \vspace{-1.5em}
    \includegraphics[width=.32\textwidth]{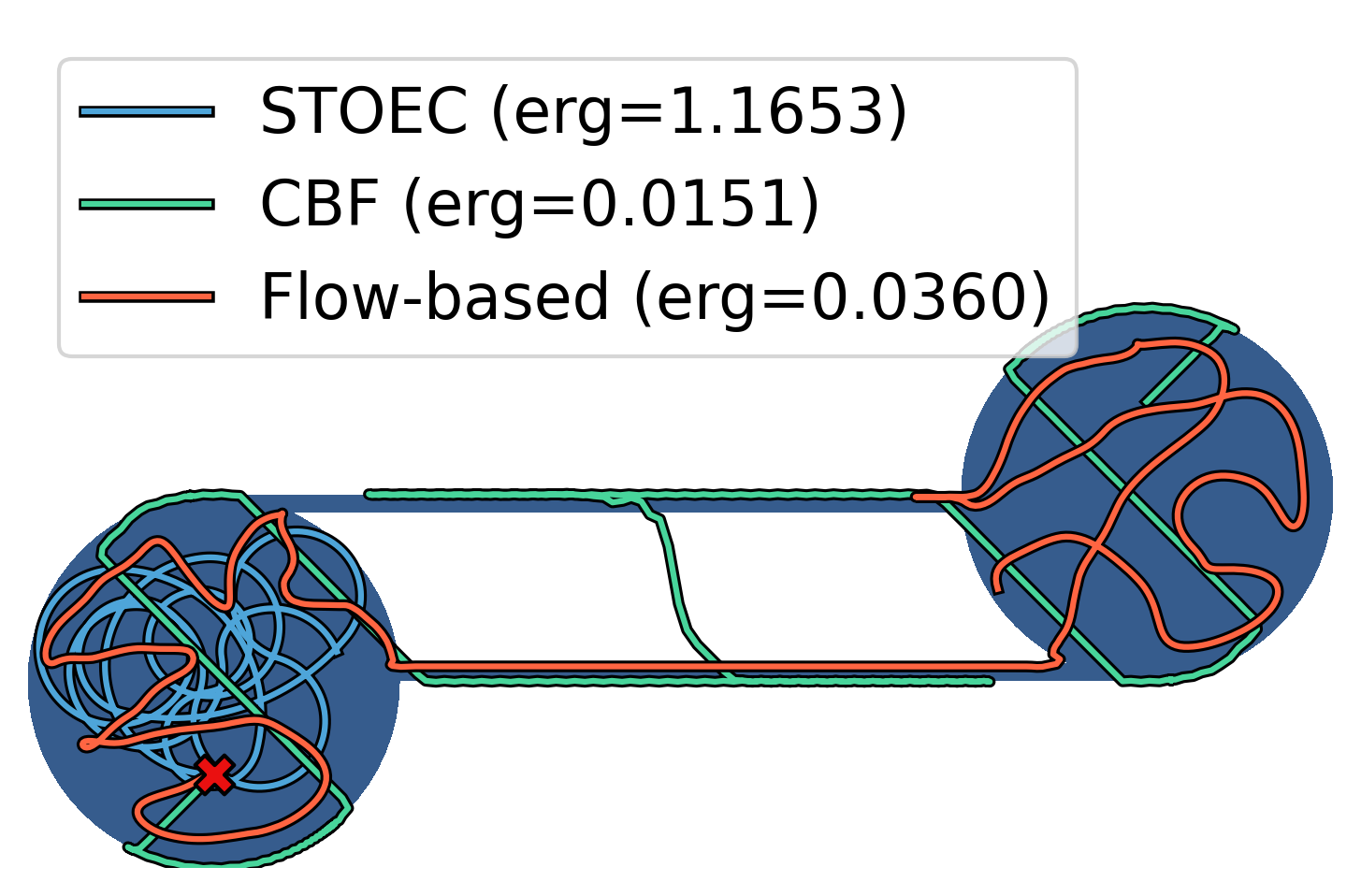}
    \includegraphics[width=.32\textwidth]{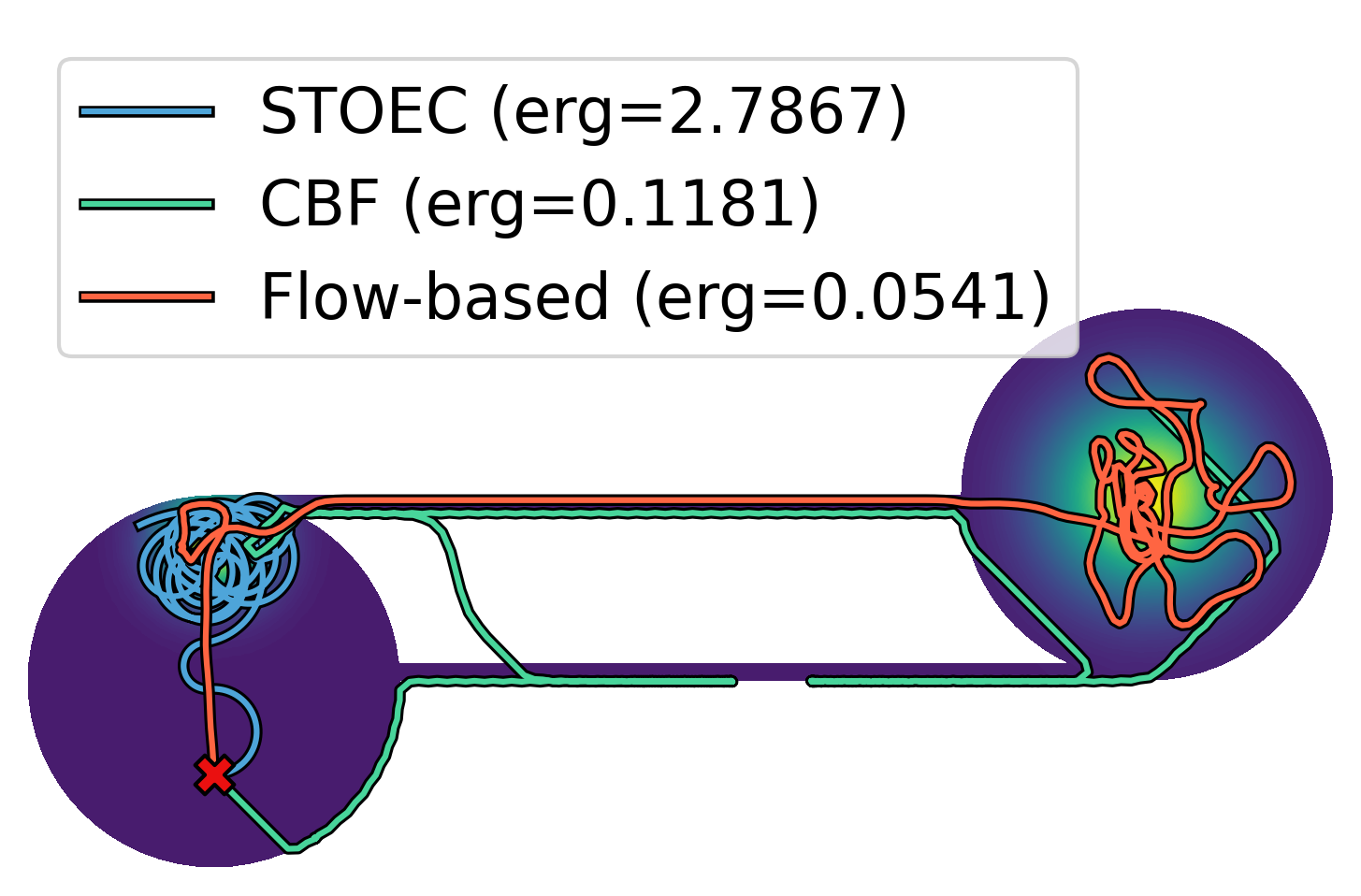}
    \includegraphics[width=.32\textwidth]{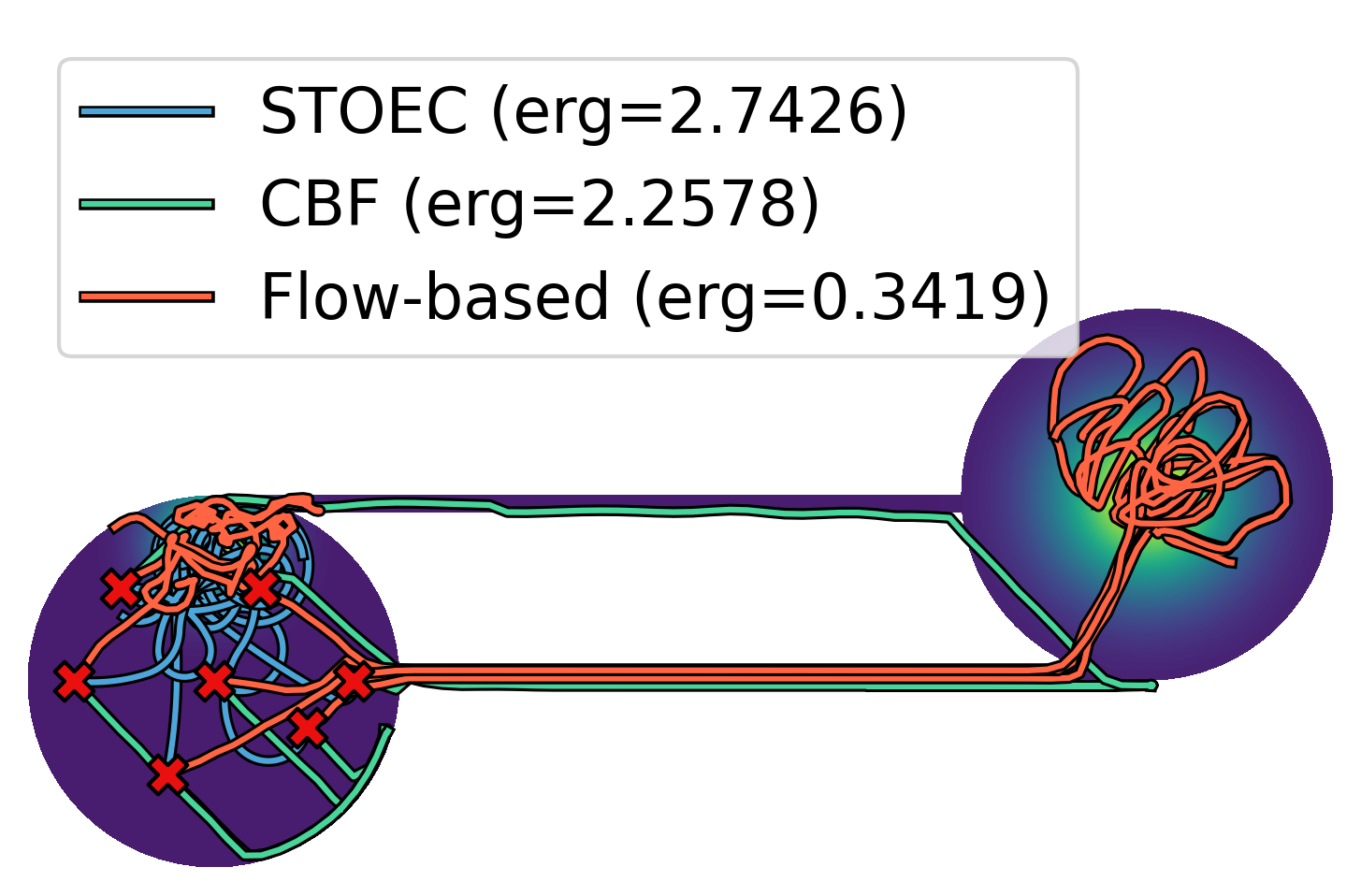}
    \caption{Results for the rooms map comparing our and prior approaches.
    From left to right, we have the trajectories generated for the uniform, nonuniform, and multi-agent nonuniform test cases.
    The ergodic metric value for each trajectory is shown in the figure legend, calculated using the Laplace-Beltrami ergodic metric.}
    \label{fig:rooms-exp}
    \vspace{-1.5em}
\end{figure}

\begin{table}[h!]
\centering
\vspace{-1.5em}
\caption{
The Fourier (F) and Laplace-Beltrami (LB) ergodic metric values for the trajectory produced by each method on the rooms map. 
When compared with two prior works, our approach shows at least an order of magnitude improvement in the ergodicity values.
}
\label{tab:rooms-results}
    \vspace{-.5em}
\begin{tabular}{|c|cccc|cccc|} \hline
    & \multicolumn{4}{c|}{Single Agent} & \multicolumn{4}{c|}{Multi-Agent} \\
    Rooms map& \multicolumn{2}{c}{Uniform} & \multicolumn{2}{c|}{Gaussian} & \multicolumn{2}{c}{Uniform} & \multicolumn{2}{c|}{Gaussian}  \\
    & F&LB & F&LB & F&LB & F&LB \\\hline
    STOEC \cite{ayvali2017ergodic} & 0.1805 & \multicolumn{1}{c|}{1.1653} & 0.4869 & 2.7867 & 0.1465 & \multicolumn{1}{c|}{1.1422}  & 0.4507 & 2.7426 \\
    Control barrier \cite{lerch2023safety} & \multicolumn{2}{c|}{{\color[HTML]{FF0000} Collision}} & \multicolumn{2}{c|}{{\color[HTML]{FF0000} Collision}} & 0.0255 & \multicolumn{1}{c|}{0.2786}  & 0.2958 & 2.2578 \\
    \rowcolor[HTML]{9AFF99} Flow-based (ours) & 0.0024 & \multicolumn{1}{c|}{0.0358} & 0.0189 & 0.0541 & 0.0116 & \multicolumn{1}{c|}{0.1140}  & 0.0441 & 0.3419 \\
    \hline
\end{tabular}
    \vspace{-3.5em}
\end{table}

\section{Conclusions and Future Work} \label{sec:conclusion}
In this work, we introduce the use of measure-preserving flows to construct feasible ergodic trajectories for point-size robots in a convoluted environment. We support this approach with mathematical proofs and show that a randomly time-varying combination of the underlying measure-preserving flows is guaranteed to yield an ergodic trajectory in the infinite time limit. Further, we present the Laplace-Beltrami ergodic metric to handle complex environments. Simulated results with single and multi-agent systems on maps that represent interior hallways and corridors illustrate the robustness of our approach. Empirical results indicate an order of magnitude improvement in ergodicity values and show success in scenarios where other approaches fail. Future directions for this work will focus on preventing collisions between finite-sized robots. Additionally, it is worth incorporating robot dynamics and physical parameters in the generation of these vector fields to help deploy generated trajectories on physical robotic systems and confirm the theoretical guarantees in the presence of physical constraints. 

%
%
\bibliographystyle{unsrt}
\bibliography{references}

\pagebreak
\appendix
\section{Proof of \cref{thm:measure-preserving-divergence}} \label{app:pf-thm1}
    \textbf{Theorem: }
    Let $(X,\mathscr B,\mu)$ be a measure space (we assume $X\subseteq \mathbb R^n$). Let $p(x)$ be the probability density function associated with $\mu$ such that $\mu(A)=\int_A p(x)dx$.

    The set of all measure-preserving flows on this measure space can be described through their associated vector fields $T = \exp(\vec v)$. Furthermore, these vector fields form a linear subspace described by the following partial differential equation.
    \begin{equation}
      \nabla\cdot (p(x) \vec{v}(x)) = 0
    \end{equation}

\begin{proof}
    Consider a flow $T^{dt}(x)$ applied for an infinitesimal duration $dt\to 0$. As the associated vector field $v(x)$ is essentially the time-derivative of the flow, we can write as $dt\to 0$,
    \begin{equation}
        T^{dt}(x) \approx x + v(x) dt.
    \end{equation}

    Second, recall the criterion for measure-preserving: that $\mu(A) = \mu(T^{-dt}(A))$ for all subsets $A\subseteq X$ and all $dt$.
    To prove this, we only need to show this equality for $\delta$-balls $A=B_{\delta}(x)$ around every point $x\in X$. In integral form, the measure-preserving criterion is written
    \begin{align} \label{eq:measure-preserve-integral}
        \int_{B_\delta(x)} p(x)dx &= \int_{T^{-dt}B_\delta(x)} p(x) dx.
    \end{align}

    To first order in $p(x)$ as $\delta\to0$, the left hand side of (\ref{eq:measure-preserve-integral}) is $\int_{B_\delta(x)} p(x)dx = p(x)\text{Vol}(B_\delta)$. As the delta ball gets smaller, the integral just evaluates the density function $p(x)$ at the test point $x$ weighted by the volume of the delta ball.

    The right hand side of (\ref{eq:measure-preserve-integral}) applies the infinitesimal flow $T^{dt}$ to the delta ball. Making a first order approximation, we assume that the flow $T^{-dt}(x')\approx x' + v(x')dt$ is linear within a small neighborhood of the test point $x$.
    \begin{align} \label{eq:Tx-linear-approx}
        T^{-dt}(x') &\approx [\nabla v(x) x - v(x)]dt + \left(I - \nabla v(x)dt \right) x'
    \end{align}

    Using (\ref{eq:Tx-linear-approx}) the translation of the center of $B_\delta(x)$ is
    \begin{align} \label{eq:translation-bdelta}
        T^{-dt}(x) = x - v(x)dt
    \end{align}

    And the change in ball volume under a linear transform is:
    \begin{equation} \label{eq:vol-bdelta}
        \text{Vol}(T^{-dt}B_\delta) = \text{Vol}(B_\delta) \text{Det}(I - \nabla v(x)dt)
    \end{equation}

    Combining (\ref{eq:translation-bdelta}) and (\ref{eq:vol-bdelta}) followed by a first-order approximation near $dt\to0$, the right hand side of (\ref{eq:measure-preserve-integral}) can be evaluated.
    \begin{align}
        \int_{T^{-dt}B_\delta(x)} p(x) dx &\approx p(x-v(x)dt) \textrm{Vol}(B_\delta)\text{Det}(I - \nabla v(x)dt) \\
        &\approx p(x)\textrm{Vol}(B_\delta)  - \nabla\cdot(p(x) v(x)) \textrm{Vol}(B_\delta)dt \label{eq:thm1-rhs-simplified}
    \end{align}

    Comparing the reduced left- and right-hand sides of (\ref{eq:measure-preserve-integral}), we can see they differ by $\nabla\cdot(p(x) v(x)) \textrm{Vol}(B_\delta)dt$. Since $\textrm{Vol}(B_\delta)$ and $dt$ are both nonzero, the equality holds if and only if $\nabla\cdot(p(x) v(x)) = 0$ for all $x\in X$. Further, this is a linear homogenous partial differential equation, so the set of solutions forms a linear subspace. \hfill $\qed$
\end{proof}

\end{document}